\documentclass[letterpaper, 10 pt, conference]{ieeeconf}  

\IEEEoverridecommandlockouts                              

\usepackage{graphics} 
\usepackage{subcaption} 
\usepackage{epsfig} 
\usepackage{amsmath} 
\usepackage{amsfonts}
\usepackage{amssymb}  
\usepackage[colorinlistoftodos,prependcaption,textsize=tiny]{todonotes}
\usepackage{algorithm}
\usepackage[nolist]{acronym}
\usepackage{algorithmicx}
\usepackage{algpseudocode}
\usepackage{import}
\usepackage{hyperref}
\usepackage{gensymb}
\usepackage{multirow} 
\usepackage[super]{nth}
\usepackage{siunitx}
\usepackage[subpreambles=true]{standalone} 
\usepackage{colortbl}
\usepackage{standalone}
\usepackage{subcaption}
\usepackage{import}
\let\labelindent\relax
\usepackage{enumitem}

\usepackage{amsmath}
\usepackage{amsfonts}
\usepackage{amssymb}
\usepackage{graphicx}
\usepackage{hyperref}
\usepackage{cprotect}
\usepackage{caption}
\usepackage{subcaption}
\usepackage{pgfplots}
\usepackage{listings}
\usepackage{tikz}
\usepackage{import}
\usepackage[subpreambles=true]{standalone}
\usepackage{tikz-3dplot}
\usepackage{tikz_utils}
\usepackage{graphicx}

\usetikzlibrary{external}

\pgfplotsset{compat=1.13}
\usetikzlibrary{math}
\usetikzlibrary{positioning}
\usepgfplotslibrary{statistics}
\usetikzlibrary{patterns,snakes} 

\DeclareMathOperator*{\argmax}{arg\,max}

\def \setReal {\mathbb{R}}
\def \setReach{\mathcal{R}_{\mathrm{dist}}}

\def \state {\boldsymbol{x}}
\def \setState {\mathcal{X}}
\def \setFree {\mathcal{X}_\mathrm{free}}
\def \setObs {\mathcal{X}_\mathrm{obs}}

\def \setControl  {\mathcal{U}}
\def \control {\mathbf{u}}
\def \start {\boldsymbol{x}_\mathrm{start}}
\def \goal {\boldsymbol{x}_\mathrm{goal}}

\def \traj {\sigma(t)}

\def \costFunction {c}

\def \setVertices {{V}}
\def \setEdges {{E}}
\def \setNearVertex {U}

\def \traj {\sigma}

\newtheorem{theorem}{Theorem}

\def\cca#1{%
  \ifdim #1 pt > 0pt%
      \pgfmathsetmacro\calc{ #1 * 50.0 / 2.0}%
      \ifdim \calc pt > 50pt
          \pgfmathsetmacro\calc{50}%
      \fi
      \edef\clrmacro{\noexpand\cellcolor{green!\calc}}%
      \clrmacro%
  \else%
      \pgfmathsetmacro\calc{ (-1) * #1 * 50.0 / 2.0}%
      \ifdim \calc pt > 50pt
          \pgfmathsetmacro\calc{50}%
      \fi
      \edef\clrmacro{\noexpand\cellcolor{red!\calc}}%
      \clrmacro%
  \fi%
  {#1}
}

\title{\LARGE \bf
Dispertio: Optimal Sampling for Safe \\Deterministic Sampling-Based Motion Planning
}

\author{Luigi Palmieri$^{1\dagger}$ and Leonard Bruns$^{2\dagger}$ and Michael 
Meurer$^{3}$ 
and Kai O. Arras$^{1}$ %
\thanks{$^{1}$L.~Palmieri and K.O.~Arras are with Robert Bosch GmbH, Corporate 
Research, Stuttgart, Germany,
	\texttt{\{luigi.palmieri, kaioliver.arras\}@de.bosch.com}.
}
\thanks{$^{2}$L.~Bruns is with RWTH Aachen, Germany and KTH Stockholm, Sweden
	\texttt{leonardb@kth.se}.
}
\thanks{$^{3}$M.~Meurer is with RWTH Aachen, Germany and German Aerospace Center (DLR), Oberpfaffenhofen, Germany
	\texttt{michael.meurer@nav.rwth-aachen.de}.
}
\thanks{$^{\dagger}$The two authors contributed equally to this work.
}
}

\begin{document}

\maketitle
\thispagestyle{empty}
\pagestyle{empty}

\begin{abstract}
A key challenge in robotics is the efficient generation of optimal robot motion 
with safety guarantees in cluttered environments. 
Recently, deterministic optimal sampling-based 
motion planners have been shown to achieve good performance towards this end, in particular in terms of planning efficiency, final solution cost, quality guarantees as well as non-probabilistic completeness.
Yet their application is still limited to relatively simple systems (i.e., 
linear, holonomic, Euclidean state spaces). In this work, we extend 
this technique to the class of symmetric and optimal driftless systems by 
presenting Dispertio, an offline dispersion optimization technique for 
computing sampling sets, aware of differential constraints, for 
sampling-based robot motion planning.
We prove that the approach, when combined with 
PRM*, is deterministically complete and retains asymptotic optimality.
Furthermore, in our experiments we show that the proposed deterministic 
sampling technique outperforms several baselines and alternative methods in 
terms of planning efficiency and solution cost.
\end{abstract}

\section{Introduction}

Motion planning is key to intelligent robot behavior. For motion planning in 
safety-critical applications, where self-driving cars, social or collaborative 
robots operate amidst and work with humans, safety guarantees, explainability 
and deterministic performance bounds are of particular interest. In the past, 
many motion planning approaches have been introduced to improve planning 
efficiency, path quality and applicability across classes of robotic systems. 
Probabilistic sampling-based motion planners 
\cite{lavalle2006planning,lavalle2001randomized,kavraki1996probabilistic} and 
their optimal variants \cite{karamanIJRR2011,janson2015fast} have shown to 
outperform combinatorial approaches \cite{lozano1979algorithm}, especially for 
high-dimensional systems with complex environments and differential constraints.
Sampling-based planners explore the configuration space by sampling states and connecting them to the roadmap, or tree, which represents and keeps track of the spatial connectivity. Typically samples are drawn from a uniform distribution over the state space by an independent and identically 
distributed (i.i.d.) random variable. Biasing techniques towards the goal 
region or promising areas of the configuration space may be used if available \cite{palmieriICRA2016,palmieriICRA2017}. The randomness of the samples set 
ensures good exploration of the configuration space, but comes at the expense 
of stochastic results which may strongly vary for each planning query in terms 
of planning efficiency and path quality. This stochasticity makes the formal 
verification and validation of such algorithms, needed for safety-critical 
applications, difficult to obtain.
\begin{figure}[t!]
	\centering
	\begin{subfigure}[b]{0.49\columnwidth}
		\centering
		\scalebox{0.6}{\includegraphics{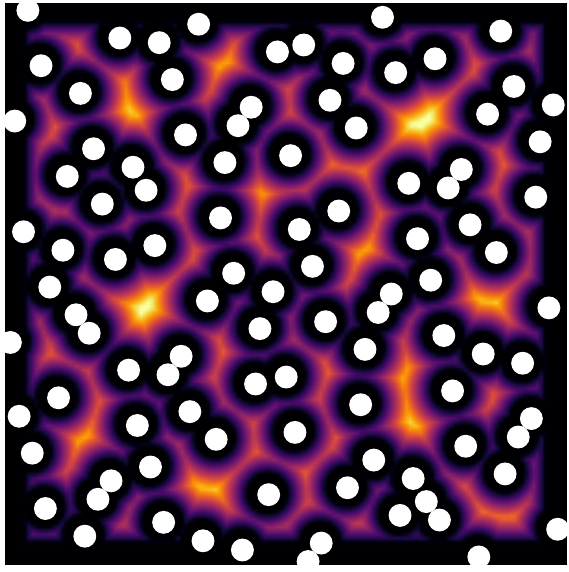}}
			\vspace{-0.048cm}
		\subcaption{Halton}
	\end{subfigure}
	\begin{subfigure}[b]{0.49\columnwidth}
		\centering
		\scalebox{0.6}{\includegraphics{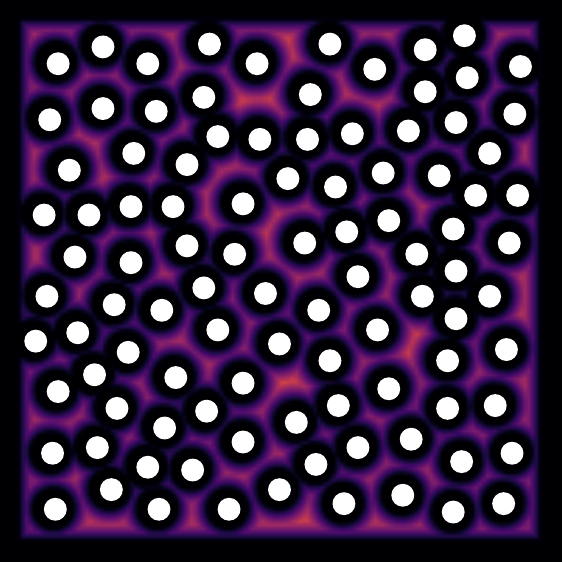}}
		\subcaption{Dispertio}
	\end{subfigure}
	\caption{Comparison of the coverage between $l_2$ low dispersion Halton 
	samples and our optimized samples for a 2D Euclidean case. Bright color 
	highlights uncovered areas. Our approach achieves better coverage than the 
	baseline.}
	\label{fig:cover_girl}
\end{figure}

To address this issue, several authors 
\cite{lavalle2004relationship,janson2018deterministic} proposed to use 
deterministic sets (or sequences). Contrarily to using i.i.d.\ random variables, 
this technique allows to achieve deterministic planning behaviors while still 
getting on par or even better performance. Moreover, as described also in 
\cite{lavalle2004relationship,janson2018deterministic}, deterministic sampling 
allows an easier certification process for the planners (e.g., in terms of final 
cost, clearance from the obstacles). Particularly, as we will see also in our case, 
those approaches have been shown to be complete (i.e., to find a solution) for 
planning queries for which a solution with certain clearance exists.  
However, current approaches limit their applicability to Euclidean spaces \cite{lavalle2004relationship}, systems with linear affine dynamics \cite{janson2018deterministic} and specific driftless ones \cite{poccia2017}.

With the goal to further enhance the usage of deterministic sampling to 
symmetric and optimal driftless systems, in this 
work we present \emph{Dispertio}, an optimization-based approach to 
deterministic sampling. 
The approach computes a sampling set which minimizes the actual dispersion of the samples. To compute the dispersion metric, we need access to a steer function \cite{palmieriIROS2014,reeds1990optimal} that can compute an optimal path connecting two states. We focus our attention on 
uninformed batch-based algorithms (e.g., PRM* \cite{karamanIJRR2011}) where the set of samples can be precomputed offline. We prove that the approach, when combined with PRM*, is 
deterministically complete and retains asymptotic optimality. Furthermore, we systematically compare our approach to the existing baselines \cite{janson2018deterministic,poccia2017}. The experiments demonstrate that our 
approach outperforms the baselines in terms of planning efficiency and overall final path quality. 

%
%
%
%
\begin{figure*}[bth!]
	\centering
	\begin{tikzpicture}
    	\useasboundingbox (0,0) rectangle (0,0);
		\draw [draw=none, fill=gray!30!white, rounded corners] (0,-0.15) rectangle ++(14.12,3.1);
	\end{tikzpicture}
	\begin{subfigure}[b]{0.19\textwidth}
		\centering
		Classical\\grid-search\\[.2cm]
		\scalebox{0.6}{\includegraphics{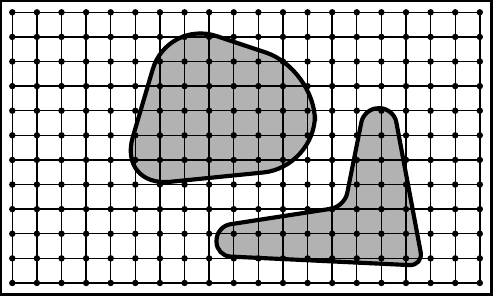}}
	\end{subfigure}
	\begin{subfigure}[b]{0.19\textwidth}
		\centering
		Subsampled\\grid-search\\[.2cm]
		\scalebox{0.6}{\includegraphics{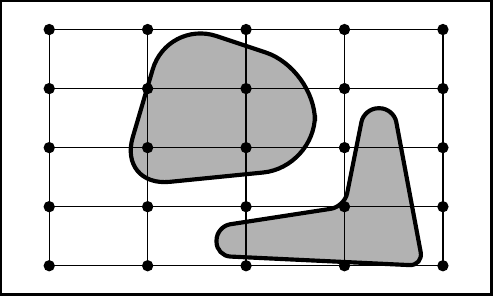}}
	\end{subfigure}
	\begin{subfigure}[b]{0.19\textwidth}
		\centering
		Lattice-based\\grid-search\\[.2cm]
		\scalebox{0.6}{\includegraphics{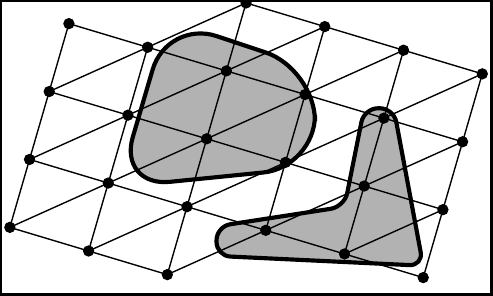}}
	\end{subfigure}
	\begin{subfigure}[b]{0.19\textwidth}
		\centering
		Quasi-random\\roadmap\\[.2cm]
		\scalebox{0.6}{\includegraphics{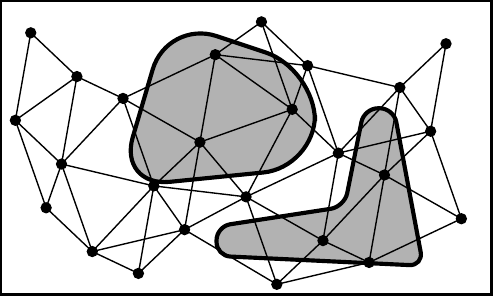}}
	\end{subfigure}
	\begin{subfigure}[b]{0.19\textwidth}
		\centering
		Probabilistic\\roadmap\\[.2cm]
		\scalebox{0.6}{\includegraphics{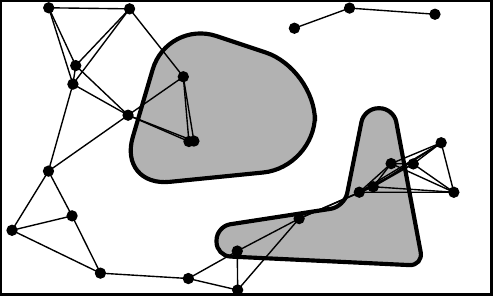}}
	\end{subfigure}
	\caption{Range of possible sampling and roadmap types as introduced by LaValle \cite{lavalle2004relationship}. The highlighted ones are deterministic.}
	\label{fig:sampling_type}
\end{figure*}
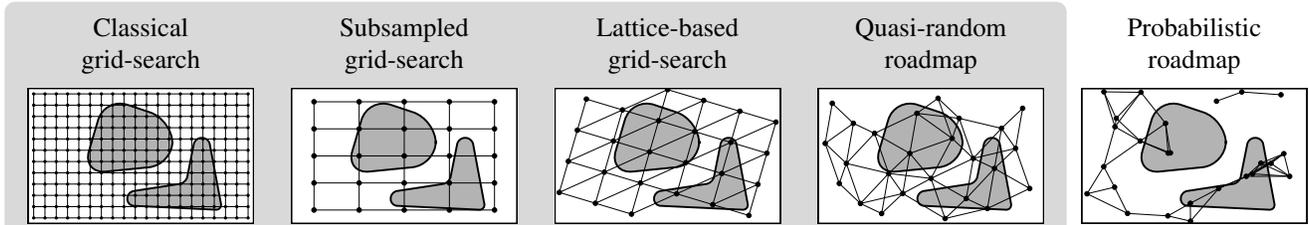
\section{Related Work}
\label{sec:related_work}
LaValle et al.~\cite{lavalle2004relationship} highlight the relationship 
between grid-based and probabilistic planning, see 
Fig.~\ref{fig:sampling_type}. 
The authors advocate that grid-based planners and probabilistic sampling-based 
planners all belong to the same class of sampling-based algorithms and are 
extremes of a broad spectrum of sampling strategies, ranging from deterministic to highly stochastic techniques. They highlight the benefits of deterministic 
sampling sets (or sequences) such as grids \cite{lengyel1990real,nashAAAI2007}, 
lattices \cite{pivtoraikoIROS2011}, or Halton and Sukharev sequences 
\cite{van1935verteilungsfunktionen,sukharev1971optimal}.
In particular, LaValle et al.~\cite{branicky2001quasi,lavalle2004relationship} show that \emph{dispersion} (see 
Sec.~\ref{sec:ourapproach} for the definition) is the {deciding} metric when it 
comes to resolution complete path planning. The reason is that dispersion 
provides lower bounds on the coverage of the space. The authors prove 
that low dispersion sampling, using for example Halton and Sukharev sequences 
\cite{van1935verteilungsfunktionen,sukharev1971optimal}, provides deterministic 
completeness guarantees on 
finding feasible paths, which i.i.d.\ sampling can 
only probabilistically provide, i.e., the planner will find a solution with a probability of 1 as the number of samples goes to infinity.
Our approach follows the ideas presented by \cite{lavalle2004relationship} and 
extends their results to motion planning with differential constraints.

While the authors in \cite{lavalle2004relationship} focus on feasibility in deterministic sampling-based motion planning, Janson et al.~\cite{janson2018deterministic} extend the approach to address optimality. 
The authors show that with a particular choice of low dispersion sampling 
($l_2$ dispersion of order $O(n^{-1/D})$, e.g., Halton sequence, with $D$ being the state space 
dimension of the considered system), optimal 
sampling-based planners (i.e.,  PRM* \cite{karamanIJRR2011,schmerling2015optimaldriftless}) can use a lower 
connection radius compared to i.i.d.\ sampling thus requiring a lower 
computational complexity, i.e., $r_n \in \omega(n^{-1/D})$. Moreover, they show 
that the cost or suboptimality of the returned solution can be bounded, based 
on the dispersion. The latter work limits its applicability to Euclidean spaces and to systems having linear affine dynamics. In comparison our method can be applied also to symmetric and optimal driftless systems with differential constraints.

Poccia \cite{poccia2017} proposes an approach for generating a set of 
deterministic samples for nonholonomic systems. The approach needs an explicit 
and careful analysis of the system equations to come up with a sampling scheme. 
Differently, our approach provides an algorithm that only needs the availability of an optimal \emph{steer} function, a common assumption for optimal sampling-based planning \cite{karamanIJRR2011,schmerling2015optimaldriftless}.

Unlike state-lattice approaches \cite{pivtoraikoIROS2011}, which can be seen as part of the class of deterministic sampling-based planners,  
our approach does not rely on a regular grid or a set of pre-defined motion primitives. Instead, it optimizes the position 
of the samples based on the dispersion metric that accounts for the differential constraints of the system.
\section{Our Approach}
\label{sec:ourapproach}
In this section, we formalize the problem that we aim to solve and the 
novel dispersion definition. We will then describe our algorithm and analyze its properties. 
%
\subsection{Problem Definition}
Let $\setState \subset \setReal^D$ be a manifold defining a configuration 
space, $\setControl \subset \setReal^M$ the symmetric control space, $\setObs \subset 
\setState$ the obstacle space and $\setFree = \setState \setminus \setObs$ the 
free space. 
A driftless control-affine system can be described by a 
differential equation as
\begin{equation}
\label{eq:smoothFunction}
\dot{\state}(t)= \sum_{j=1}^{M}g_j(\state(t))\control(t)
\end{equation}
where $\state(t) \in \setState $, $ \control(t) \in \setControl $, for all $t$, 
and $g_1, \ldots, g_M$ being the system vector fields on $\setState$. For the 
remainder of the paper we will focus on symmetric systems for 
which an optimal steer function exists.

Let $\gamma$ denote a planning query, defined by its initial state 
$\start \in \mathcal{X}$ and goal state $\goal \in 
\mathcal{X}$. We define the set of all possible 
solution paths for a given query $\gamma$ as $\Sigma_\gamma$, with $\sigma \in 
\Sigma_\gamma : [0,1] \to \mathcal{X}_\mathrm{free}$ being one of 
the possible solution paths such that $\sigma(0)=\start$ 
and $\sigma(1)=\goal$. The arc-length of a path $\sigma$ 
is defined by $l(\sigma) = \int_{0}^{1} ||\dot{\sigma}(t)||_2~dt$. The 
arc-length 
induces a \emph{sub-Riemannian} distance $\mathrm{dist}$ on $\setState$: 
$\mathrm{dist}(\state,\boldsymbol{z})=\inf_\sigma l(\sigma)$, i.e., the length 
of the optimal path connecting $\state$ to 
$\boldsymbol{z}$, which due to our assumptions is also symmetric. Let $\sigma^*$ denote the set of all points along a path 
$\sigma$. The $\mathrm{dist}$-clearance of a path $\sigma$ is defined as 
\begin{equation}
\delta_\mathrm{dist}(\sigma)=\sup \big\{r\in \mathbb{R}~|~
\setReach({\state}, r)\subseteq \mathcal{X}_\mathrm{free} 
~\forall {\state}\in \sigma^*\big\}
\end{equation}
where $\setReach({\state}, r)$ is the cost-limited reachable set  (closed if not 
otherwise stated) for the system in Eq.~\ref{eq:smoothFunction} centered at ${\state}$ 
within a path length of $r$ (e.g., a sphere for Euclidean systems):
\begin{equation}
\label{eq:reachableset}
\setReach(\state, r) = \big\{\boldsymbol{z} \in \setState~| 
~\mathrm{dist}(\state, \boldsymbol{z}) \leq r\big\}.
\end{equation}
The $\mathrm{dist}$-clearance of a query $\gamma$ is defined as 
\begin{equation}
\delta_\mathrm{dist}(\gamma)=\sup \big\{\delta_\mathrm{dist}(\sigma)~|~\sigma \in 
\Sigma_\gamma\big\}
\end{equation}
and denotes the maximum clearance that a solution path to a query can have.
An optimal sampling-based 
algorithm solves the 
following $\hat{\delta}_\mathrm{dist}$-robustly feasible motion planning problem $\mathcal{P}$: given a query $\hat{\gamma}$ with a $\mathrm{dist}$-clearance of $\delta_\mathrm{dist}(\hat{\gamma})>\hat{\delta}_\mathrm{dist}$,  
find a control $\control(t) \in \setControl$ with domain $[0,1]$ such 
that the unique trajectory $\traj$ satisfies Eq.~\ref{eq:smoothFunction}, is 
fully contained in the free 
space $\setFree \subseteq \setState$ and goes from $\start$ to $\goal$. Moreover it minimizes, asymptotically, a defined cost function 
$\costFunction:\Sigma_\gamma\rightarrow\setReal_{\geq 0}$.
Hereinafter, we will use the term \emph{steer} function to indicate a 
function that generates a path in $\setState$ connecting two specified states. 
In particular we will use steer functions that solve an optimal control 
problem, i.e., minimizing the cost $\costFunction$.

In the following sections, we describe the approach to solve $\mathcal{P}$ by using an 
optimization-based sampling technique that minimizes the actual 
dispersion of the sampling set used by batch-processing algorithms (e.g., 
PRM*\footnote{Due to space limitations, we will not detail the algorithm PRM*. 
A reader interested to the properties of the algorithm can refer to 
\cite{karamanIJRR2011}.}, see Alg.~\ref{alg:prmstar}).

\subsection{Dispersion for Differentially Constrained Systems}
We use and modify the dispersion definition for a sampling set $\mathcal{S} = \left\{ 
\state_0, \state_1, \dots, \state_n \right\} \subset \setState$, introduced by 
Niederreiter \cite{niederreiter1992random} and also adopted by 
\cite{lavalle2004relationship,janson2018deterministic}:
\begin{equation}
d_\mathrm{dist}=\sup\{r>0~|~\exists \boldsymbol{x}\in \mathcal{X} \text{ with } 
\setReach(\boldsymbol{x}, r) \cap \mathcal{S}=\emptyset\} \label{eq:dispersion}.
\end{equation}
Intuitively the dispersion can be considered as the radius of the largest (open) ball 
(i.e., size of the reachable set) that does not contain an element of $S$. In 
the context of differentially constrained motion planning, we propose to adjust 
the dispersion 
metric to explicitly require the reachable sets 
$\setReach(\boldsymbol{x}, r)$ 
to be fully contained in $\mathcal{X}$:
\begin{equation}
\label{eq:dispersionkinodyn}
	\begin{split}
		\tilde{d}_\mathrm{dist}=\sup\{r>0~|~&\exists \boldsymbol{x}\in 
		\mathcal{X} \text{ with } \setReach(\boldsymbol{x}, r) \cap 
		\mathcal{S}=\emptyset \\ & \land \setReach(\boldsymbol{x}, r) \subseteq 
		\mathcal{X}\}.
	\end{split}
\end{equation} 
We also require this metric to respect possible identifications of the 
configuration space. 
Differently from previous approaches 
\cite{lavalle2004relationship,poccia2017,janson2018deterministic}, we will 
compute the dispersion metric by 
numerically computing offline the reachable sets 
$\setReach(\boldsymbol{x}, r)$ where $r > 0$ 
is the path length obtained by an optimal controller. 


\subsection{The Dispersion Optimization Algorithm}
As discussed by \cite{lavalle2001randomized,janson2018deterministic} 
multi-query sampling-based planners, such as PRM* or FMT*, 
generate as initial step a set $S$ of collision free samples, see line 2 of 
Alg.~\ref{alg:prmstar}. 
Instead of using i.i.d.\ random variables, or an existing deterministic 
technique to generate $\mathcal{S}$ (e.g., Halton sequence, 
\cite{lavalle2004relationship,poccia2017,janson2018deterministic}), we propose 
to compute the set by minimizing the dispersion of Eq.~\ref{eq:dispersionkinodyn}. Our algorithm named 
{Dispertio} is outlined in Alg.~\ref{alg:disperspion_optimization}. The 
general idea of the algorithm is to pick in each step the sample (up to 
$n < N_{CS}$) that maximizes the distance to both the defined 
border of the configuration space as well as to the next sample. In other words 
we want to greedily put the sample into the position that currently defines the 
dispersion metric. 

We propose to make this task computationally feasible by discretizing the configuration space 
into a fine grid of $N_\mathrm{CS}$ equidistant (distance could be different 
per dimension) cells. The dispersion tensor $\boldsymbol{D}$ keeps track of the 
minimum distance to either the border or closest sample for each grid cell (in 
Alg.\ref{alg:disperspion_optimization} we denote the dispersion value at the 
cell or position $\boldsymbol{c}$ as $D_{\boldsymbol{c}}$), computed by 
solving Eq.~\ref{eq:dispersionkinodyn} using an optimal steer function. 

If it is possible to compute the distance to the border quickly (e.g., 
Euclidean case), we initialize $\boldsymbol{D}$ with the distance to 
the border for each grid cell, otherwise $\boldsymbol{D}$ is initialized with $\infty$, line 1 of Alg.~\ref{alg:disperspion_optimization}. 
In this case, we check whether the update step to a potential sample would affect any border sample. If 
this is the case, we will not add the sample to $\mathcal{S}$, but instead 
run an update step on the border sample without adding it.

At each algorithm iteration, we generate 
a sample $\state_i$ that maximizes the current dispersion tensor $\boldsymbol{D}$ and add it to $\mathcal{S}$, see lines 3--7 
of Alg.~\ref{alg:disperspion_optimization}. For a given sample, $\boldsymbol{D}$ 
is updated (line 5 of Alg.~\ref{alg:disperspion_optimization}) 
with a flood-fill algorithm, by only expanding cells for which the dispersion has 
been updated. In this way we are exploiting the connectedness of time-limited reachable sets. The sequence for the flood-fill algorithm can be pre-computed to 
prevent double checking of already tested cells.

Despite having a time complexity exponential in dimensions due to the 
flood-fill algorithm \big(i.e., $\mathcal{O}(n\xi^D)$, with the constant $\xi>0$ being related to discretization and complexity of $\mathrm{dist}$), the algorithm is a 
feasible pre-computation step for many systems (e.g., Reeds-Shepp space, 6D 
kinematic chain using Euclidean distance).
Once the set $\mathcal{S}$ has been generated, we can then use it in a motion 
planning algorithm such as PRM* (Alg.~\ref{alg:prmstar}). PRM*-edges 
are generated with the same steer function used to optimize the set $\mathcal{S}$.
\begin{algorithm}[t]
	\footnotesize
	\caption{PRM*. $\start$ is the start state, $\goal$ the goal state, $n$ 
		the desired number of samples.}\label{alg:prmstar}
		\begin{algorithmic}[1]
		\Procedure{PRM*}{}
			\State $\mathcal{S} \leftarrow$ \Call{SampleFree}{$n$}
			\State $\setVertices \leftarrow \{\start, \goal\} \cup \mathcal{S}$
      		\For {$v \text{ in } \setVertices$}
	      		\State $\setNearVertex \leftarrow$ \Call{Near}{$\setVertices, v, r_{|\setVertices|}$}
   	  				\For {$u \text{ in } \setNearVertex$}
       					\If{ \Call{CollisionFree}{$v,u$}}
    						\State{$\setEdges \leftarrow \setEdges \cup \{(v,u)\}$}
				    	\EndIf	     
       				\EndFor
      		\EndFor
      		\State return \Call{ShortestPath}{$\start, \goal, (\setVertices, \setEdges)$}
		\EndProcedure
		\end{algorithmic}
\end{algorithm}
\begin{algorithm}[t]
	\footnotesize
	\caption{Dispersion Optimization}\label{alg:disperspion_optimization}
	\begin{algorithmic}[1]
		\Procedure{Dispertio}{}
		\State{$\boldsymbol{D} \gets $} \Call{DistanceToBorder}{}
		\While{$|\mathcal{S}|<n$}
		\State{$\boldsymbol{x}_i\gets \argmax_{\boldsymbol{c}}D_{\boldsymbol{c}}$}
		\State \Call{UpdateDistanceMatrix}{$\boldsymbol{D},\boldsymbol{x}_i$}
		\State{$\mathcal{S}\gets \mathcal{S} \cup \{ \boldsymbol{x}_i \}$}
		\EndWhile
		\EndProcedure
	\end{algorithmic}
\end{algorithm}

\subsection{Dispertio-PRM* Analysis}
In this section we detail how PRM*~\cite{karamanIJRR2011}, when using our deterministic sampling
approach, retains the completeness and asymptotic optimality properties as in 
\cite{lavalle2004relationship,poccia2017,janson2018deterministic}.
\begin{figure}[ht!]
	\centering
	\scalebox{0.95}{\includegraphics{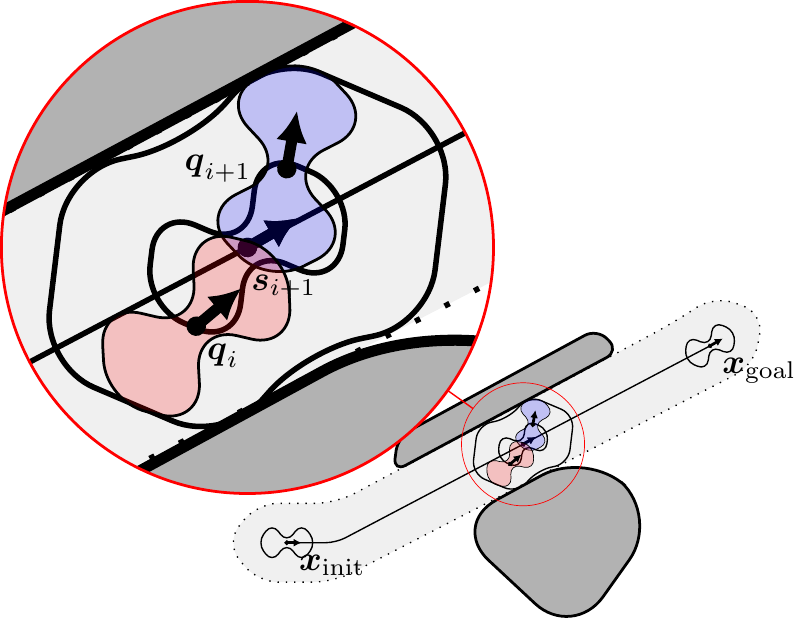}}
	\caption{Visualization of the proof of completeness. $\boldsymbol{s}_{i+1}$ 
		is placed along the (unknown) path of maximum clearance $\sigma$ such 
		that 
		$\boldsymbol{q}_i$ lies on the border of 
		$\setReach(\boldsymbol{s}_{i+1}, 
		\tilde{d}_\mathrm{dist})$. Since $\setReach(\boldsymbol{q}_i, 
		\tilde{d}_\mathrm{dist})$ and $\setReach(\boldsymbol{q}_{i+1}, 
		\tilde{d}_\mathrm{dist})$ overlap and are fully contained in 
		$\setReach(\boldsymbol{s}_{i+1}, 2\tilde{d}_\mathrm{dist})$ the path 
		from 
		$\boldsymbol{q}_i$ to $\boldsymbol{q}_{i+1}$ is 
		collision free.}\label{fig:proof}
\end{figure}
\subsubsection{Completeness}
\label{sec:completeness}
We show that the approach deterministically returns a solution if it exists and 
returns failure otherwise \cite{lavalle2004relationship,janson2018deterministic,poccia2017}. 
Note that this is a 
stronger property than \emph{probabilistic completeness} 
\cite{lavalle2006planning}.
\begin{theorem}
	Given a set of samples $\mathcal{S}$ with known dispersion
	$\tilde{d}_\mathrm{dist}$ and considering general driftless systems for which we have steer functions that are optimal and 
	symmetric, we can solve all planning queries $\gamma$ with 
	Alg.~\ref{alg:prmstar} using a connection radius $r>2\tilde{d}_\mathrm{dist}$ having 
	clearance of
	\begin{equation}
	\label{eq:desiredclearance}
	\delta_\mathrm{dist}(\gamma) > 2\tilde{d}_\mathrm{dist}.
	\end{equation}
\end{theorem}
\begin{proof}
	To see this, first note that $\setReach(\boldsymbol{x}, r)$ for optimal 
	steering functions, is equivalent to time-limited reachable sets of the system. 
	Hence, trajectories from $\boldsymbol{x}$ to any other point in 
	$\setReach(\boldsymbol{x}, r)$ will also be fully contained in 
	$\setReach(\boldsymbol{x}, r)$. Given a query $\gamma$ with clearance 
	$\delta(\gamma)>2 \tilde{d}_\mathrm{dist}$, there exists a solution 
	$\sigma$ with $\mathcal{R}(\boldsymbol{s}_i,2 \tilde{d}_\mathrm{dist})\in 
	\mathcal{X}_\mathrm{free},~\forall \boldsymbol{s}_i \in \sigma^*$. First 
	note that due to the dispersion definition, there 
	must be a sample of $\mathcal{S}$ in both 
	$\setReach(\start,  \tilde{d}_\mathrm{dist})$ and 
	$\setReach(\goal,  \tilde{d}_\mathrm{dist})$. Thus it is 
	possible to connect the start and goal configuration to the roadmap. It remains 
	to show that a $\mathrm{dist}$-clearance of $\delta_\mathrm{dist}(\gamma) > 2\tilde{d}_\mathrm{dist}$ is sufficient to find a path from $\start$ to $\goal$. Let 
	$\boldsymbol{q}_0$ and $\boldsymbol{q}_N$ denote the samples that 
	$\start$ and $\goal$ are connected 
	to, respectively. By taking $\boldsymbol{s}_1$ along a path $\sigma$ and such that $\boldsymbol{q}_0$ lies on the border of 
	$\setReach(\boldsymbol{s}_1, \tilde{d}_\mathrm{dist})$ we see, due to the dispersion definition in Eq.~\ref{eq:dispersionkinodyn}, that there must be another sample in the reachable set, denoted 
	by $\boldsymbol{q}_1$. At this point we only know that the path from 
	$\boldsymbol{q}_0$ over $\boldsymbol{s}_1$ to $\boldsymbol{q}_1$ must be 
	collision free.  Since only $\boldsymbol{q}_0$ and $\boldsymbol{q}_1$ are known, we require a factor of 2 
	in the clearance (i.e., $2\tilde{d}_\mathrm{dist}$ in Eq.~\ref{eq:desiredclearance}), which ensures that the path from $\boldsymbol{q}_0$ to $\boldsymbol{q}_1$ must be collision-free. To see this, note that due the system symmetry both 
	$\setReach(\boldsymbol{q}_\mathrm{0}, \tilde{d}_\mathrm{dist})$ and 
	$\setReach(\boldsymbol{q}_\mathrm{1}, \tilde{d}_\mathrm{dist})$ must be contained 
	in $\mathcal{R}(\boldsymbol{s}_1, 2\tilde{d}_\mathrm{dist})\subset \mathcal{X}_\mathrm{free}$. We also know that the intersection
	$\setReach(\boldsymbol{q}_\mathrm{0}, \tilde{d}_\mathrm{dist}) \cap 
	\setReach(\boldsymbol{q}_\mathrm{1}, \tilde{d}_\mathrm{dist})$ contains $\boldsymbol{s}_1$ and is thus nonempty. The trajectory from $\boldsymbol{q}_0$ 
	to $\boldsymbol{q}_1$ must pass through this intersection and is hence 
	collision-free. The same idea can now be repeated until the path to 
	$\boldsymbol{q}_N$ is found. Fig.~\ref{fig:proof} visualizes the proof.
\end{proof}

\begin{figure*}[t!]
	\centering
	\begin{subfigure}[b]{0.24\textwidth}
		\centering
		\scalebox{0.6}{\includegraphics{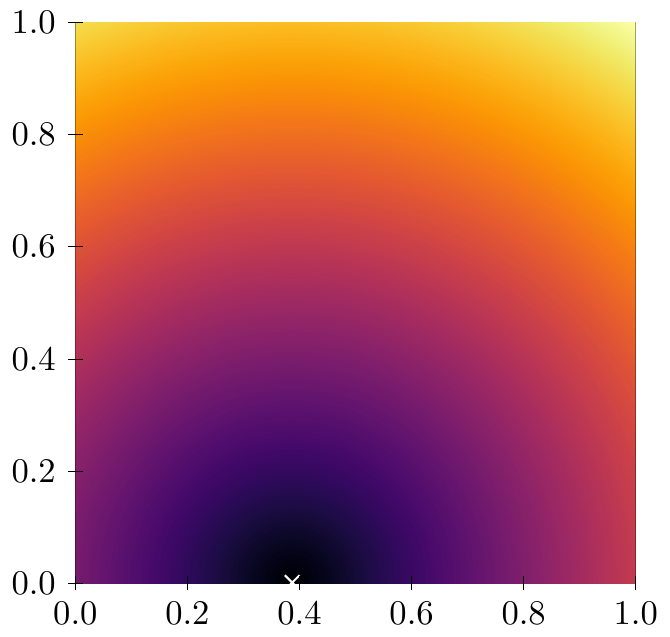}}
	\end{subfigure}
	\begin{subfigure}[b]{0.24\textwidth}
		\centering
		\scalebox{0.6}{\includegraphics{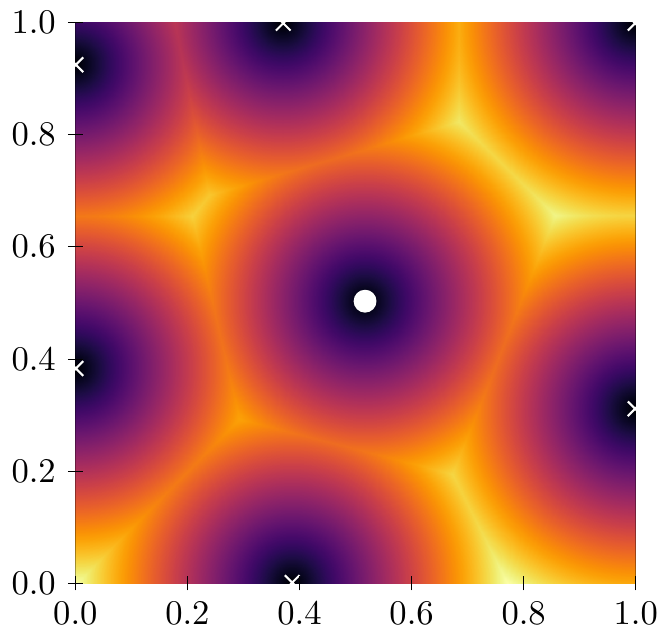}}
	\end{subfigure}
	\begin{subfigure}[b]{0.24\textwidth}
		\centering
		\scalebox{0.6}{\includegraphics{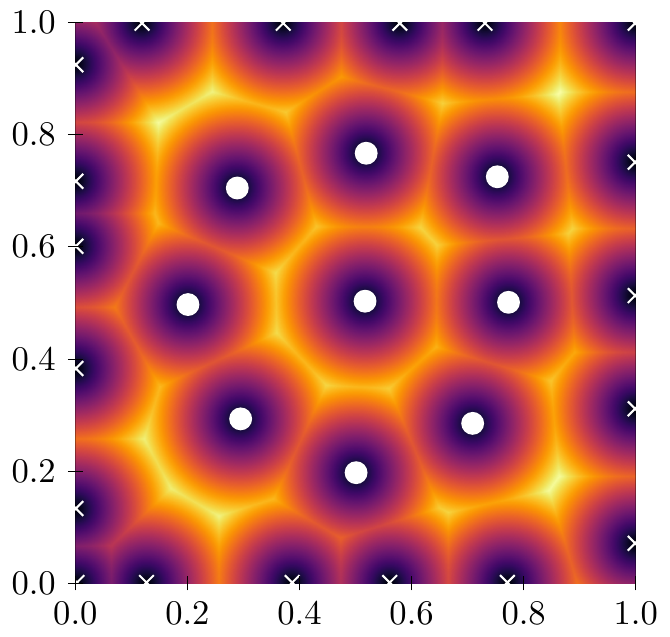}}
	\end{subfigure}
	\begin{subfigure}[b]{0.24\textwidth}
		\centering
		\scalebox{0.6}{\includegraphics{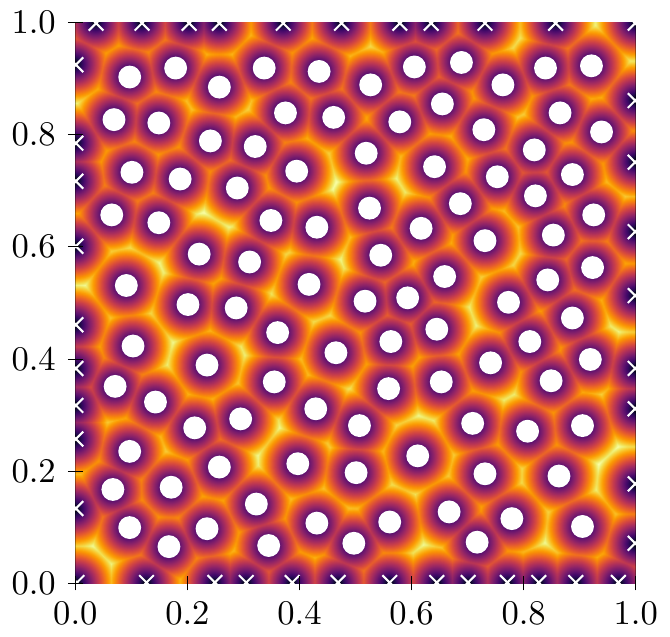}}
	\end{subfigure}
	\caption{Progression of the algorithm in 2D Euclidean space. The background 
	color indicates the distance to the next sample (i.e., the distance matrix 
	$\boldsymbol{D}$). The white crosses and dots show the processed border 
	points and actual samples respectively.}
	\label{fig:algorithm}
\end{figure*}
\subsubsection{Asymptotic Optimality}
\label{sec:asymoptimal}
In this section, following \cite{janson2018deterministic} we will show that 
PRM* is asymptotically optimal when using the sampling 
sets generated by {Dispertio}. Particularly, Janson et al.~\cite{janson2018deterministic} show 
that PRM* is asymptotically optimal when using deterministic sampling sets in $D$ 
dimensions whose dispersion is upper-bounded by $\gamma~n^{-1/D},~ \gamma>0$.
Next, we will show how the sampling sets generated by our approach reach the 
same asymptotic dispersion (see Theorem \ref{theorem:dispersion}), i.e., lower $l_2$-dispersions, for all driftless 
control-affine systems, therefore retaining the PRM* asymptotic optimality.
For the special Euclidean case, we show that the algorithm reaches the same 
asymptotically optimal dispersion as for example the Halton sequence. Note that 
for simplicity we are using a simplified version of the algorithm, without 
discretization and assuming the distance to the border is known. Throughout the discussion we again assume 
that the distance function $\mathrm{dist}$ is symmetric and optimal. Also note 
that $\mathcal{R}(\boldsymbol{x},r)$ now denotes the \emph{open} ball of radius $r$ 
at $\boldsymbol{x}$ and $V(\cdot)$ the volume of a set.
\begin{theorem}
	\label{theorem:dispersion}
	Under the assumption that the discretization of the space does not influence the placement of the samples, Alg.~\ref{alg:disperspion_optimization} dispersion can be bounded by 
	\begin{equation}
		n V(\mathcal{R}(\boldsymbol{x},d_n / 2)) \leq V(\mathcal{X})
	\end{equation}
	where $d_n$ denotes the dispersion defined for a distance function 
	$\mathrm{dist}$ as in Eq. \ref{eq:dispersion}, when $n$ samples have been 
	picked, i.e., $|\mathcal{S}|=n$. This yields for the $D$-dimensional 
	Euclidean case an asymptotic behavior of 
	\begin{equation}
		d_n \in  \mathcal{O}\left(n^{-1/D}\right)
	\end{equation} and the driftless control-affine case 
	\begin{equation}
		d_n \in \mathcal{O}\left(n^{-1/\tilde{D}}\right)
	\end{equation}
	with $\tilde{D}=\sum_{i=1}^D w_i$, where $w_i$ are the weights of the boxes 
	approximating the reachability space for driftless control-affine systems (see ball-box 
	theorem \cite{montgomery2002tour,chow2002systeme}).
\end{theorem}
\begin{proof}
	To prove the asymptotic behavior of the algorithm, let us consider the case 
	in which the discretization of the space has no effect on the placement of 
	samples\footnote{For brevity, 
	we remove the explicit $\mathrm{dist}$ from the dispersion, but it is 
	implied to be the distance function used in the algorithm and the reachable 
	sets $\mathcal{R}$.}. The key argument to analyze the asymptotic behavior 
	of the algorithm is to realize that the $n$th sample is, by construction, 
	placed such that its distance to the closest neighbor is $d_{n-1}$.  Due to 
	that, after $n$ samples have been picked, we can note that
	\begin{equation}
		d_{n-1}\leq \min_{\boldsymbol{y}\in\mathcal{S}\setminus\boldsymbol{x}}\mathrm{dist}(\boldsymbol{x},\boldsymbol{y})\leq 2 d_{n-1} \quad\forall\boldsymbol{x}\in\mathcal{S},
	\end{equation}
	where the second inequality follows from the symmetry and optimality 
	assumption of $\mathrm{dist}$. From the first inequality it follows that 
	(note that the ball is open)
	\begin{equation}
		\mathcal{R}(\boldsymbol{x}, d_{n-1})\cap \mathcal{S}=\emptyset\quad\forall \boldsymbol{x}\in \mathcal{S}.
	\end{equation}
	In addition, because of symmetry and optimality, the intersection of all open 
	balls of radius $d_{n-1}/2$ must be empty, i.e.,
	\begin{equation}
		\bigcap\limits_{\boldsymbol{x}\in \mathcal{S}}\mathcal{R}(\boldsymbol{x}, d_{n-1}/2)=\emptyset.
	\end{equation}

	Note that $d_n\leq d_{n-1}$ and with $n$ samples being in $\mathcal{S}$ we 
	can state that
	\begin{equation}
		n V\big(\mathcal{R}(\boldsymbol{x}, d_{n}/2)\big) \leq V(\mathcal{X})\label{eq:ineq}
	\end{equation}
	must hold. To upper bound the dispersion for a number of samples $n$ we 
	would optimally use an explicit term for the volume 
	$V\big(\mathcal{R}(\boldsymbol{x}, d_{n}/2)\big)$, but if no such term 
	exists (as for general sub-Riemannian balls), we need to use a lower bound, 
	for example by using the ball-box theorem. Let us first consider the case of a 
	$D$-dimensional Euclidean space $\mathcal{X}$. In that case we obtain
	\begin{equation}
		n \alpha d_n^D \leq V(\mathcal{X})
	\end{equation}
	and thus
	\begin{equation}
		d_n \leq \frac{V(\mathcal{X})^{1/D}}{\alpha^{1/D} n^{1/D}} \in 
		\mathcal{O}\left(n^{-1/D}\right)
	\end{equation}
	with $\alpha>0$. This shows that in the Euclidean case, the achieved 
	asymptotic 
	dispersion is the same as for $l_2$ low dispersion sequences (e.g., Halton). 
	For the driftless control-affine case we can use 
	the same argumentation as in \cite{poccia2017}. Under the assumption 
	that the system 
	is sufficiently regular we can find a parameter $A_\mathrm{max}$ such that 
	\begin{equation}
		\mathrm{Box}^w\left(\boldsymbol{x},\frac{d_n}{2 A_\mathrm{max}}\right)\subseteq \mathcal{R}(\boldsymbol{x}, d_{n}/2)
	\end{equation}
	and according to Lemma II.2 by Schmerling et al. 
	\cite{schmerling2015optimaldriftless} the volume is given by
	\begin{equation}
		V\Bigg(\mathrm{Box}^w\left(\boldsymbol{x},\frac{d_n}{2 A_\mathrm{max}}\right)\Bigg)=\left(\frac{d_n}{2 A_\mathrm{max}}\right)^{\tilde{D}}
	\end{equation}
	with $\tilde{D}=\sum_{i=1}^D w_i$. We can rewrite Eq.~\ref{eq:ineq} as
	\begin{equation}
		n \left(\frac{d_n}{2 A_\mathrm{max}}\right)^{\tilde{D}} \leq V(\mathcal{X})
	\end{equation}
	and thus 
	\begin{equation}
		d_n \leq \frac{V(\mathcal{X})^{1/\tilde{D}} 2 A_\mathrm{max} }{ n^{1/\tilde{D}} } \in \mathcal{O}\left(n^{-1/\tilde{D}}\right).
	\end{equation}
\end{proof}

Note that if the number of samples approaches the discretization of the space, 
they will actually converge to a Sukharev grid \cite{sukharev1971optimal}. 
Hence, in the Euclidean case, the asymptotic 
dispersion is still $\mathcal{O}\left(n^{-1/D}\right)$, but in the general 
case, we would need to inner-bound the reachable set with a Euclidean ball, 
which would lead to rather crude approximations as shown by Janson et al.~\cite{janson2018deterministic} for the linear affine case. Thus, especially for 
nonholonomic systems, a grid of high resolution may be important to capture the shape of the reachable sets. Fig.~\ref{fig:dispersion_progress} numerically compares the dispersion for the Reeds-Shepp case after $n$ samples for i.i.d., Halton and the proposed approach.

Given that our set $S$ has the same asymptotic dispersion as $l_2$ low 
dispersion 
sequences, our approach retains the asymptotic analysis carried out in 
\cite{janson2018deterministic,poccia2017} and it allows the usage of a PRM* 
connection radius $r_n \in \omega(n^{-1/D})$.
\section{Evaluation}
\label{sec:evaluation}

In this section we describe the experiments to evaluate how our 
approach performs in terms of planning efficiency and path quality compared to a set of 
baselines.

To this end, we design two main 
experiments. In the first experiment, we compare our approach against the 
baselines (uniform i.i.d.\ samples, Halton samples 
\cite{lavalle2004relationship}, 
Poccia's approach \cite{poccia2017}, state-lattice approach 
\cite{pivtoraikoIROS2011}) for a car-like kinematic systems (i.e., Reeds-Shepp (RS)
\cite{reeds1990optimal}), 
over a subset of maps from the benchmark \emph{moving-ai} 
\cite{sturtevant2012benchmarks}, i.e., city maps, see Fig.~\ref{fig:qualitativeresults} for example maps. The benchmark contains maps with several narrow corridors, and the planner needs to perform complex maneuvers (i.e., fully exploiting the full maneuverability), to let the car achieve its goal.
We use a minimum turning radius $\rho=5\textrm{~m}$ and plan in environments of 
different size with $w$ being the width of the map. For the state 
lattice, sets of motion primitives have been chosen after an informal 
validation and are shown in Fig.~\ref{fig:primitives}. The actual dispersion 
is reported as $\tilde{d}_\mathrm{rs}$ and the number of drawn samples as $n_\mathrm{all}$. 

In the {second experiment}, we compare the approach to the baselines 
on a set of randomized maps and random planning queries. 
To show the general applicability of the algorithm we also 
benchmark it for a 6D kinematic chain in the 2D plane (comparing it to Halton 
sequence and i.i.d.\ samples). In this case, each joint either has an angle 
$\theta_i\in [-3,3]$ with $i=1,...,6$, or we plan in an identified space (i.e., the arms can wrap around) with $\theta_i\in [-\pi,\pi] / \sim$. We use as distance function the 2-norm 
in joint space (respecting the possible identification).

In both experiments we evaluate the approach in terms of cost (i.e., path 
length) and success rate, and show planning efficiency by plotting 
the cost progression. 
Additionally we compare the trend for the dispersion 
(Eq.~\ref{eq:dispersionkinodyn}) conditioned on the number of samples for the car-like kinematics, obtained 
by our approach and the different baselines, see 
Fig.~\ref{fig:dispersion_progress}.

We use OMPL \cite{sucan2012open} and adopt its PRM* implementation (we 
made it deterministic by removing the random walk expansion step), with the 
default $k$-nearest connection strategy of $k_n = e\left(1+1/D\right)\log n$, 
which ensures asymptotic optimality for all the samplers. 
The same nearest neighborhood search and collision checkers (only different for 
examined systems) are used for all experiments and they are run on a machine 
with Xeon E5-1620 CPU and 32GB of RAM.
\begin{figure}[t!]
	\centering
	\scalebox{0.9}{\includegraphics{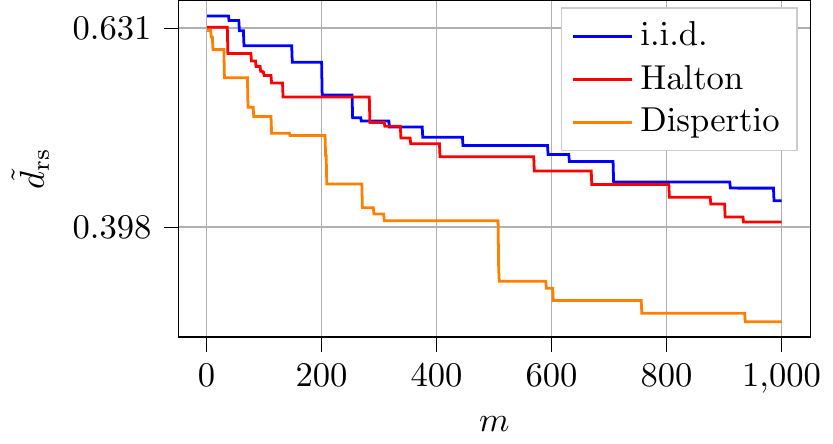}}
	\caption{Dispersion trend for the Reeds-Shepp 
	case ($\eta=1.0$, obstacle free environment). Our approach obtains a better 
	dispersion than the baselines, thus achieving a better coverage of the 
	state space as also shown in Fig.~\ref{fig:cover_girl}.}\label{fig:dispersion_progress}
\end{figure}
\begin{figure}[t!]
	\centering
	\begin{subfigure}[b]{0.49\columnwidth}
		\centering
		\scalebox{0.7}{\includegraphics{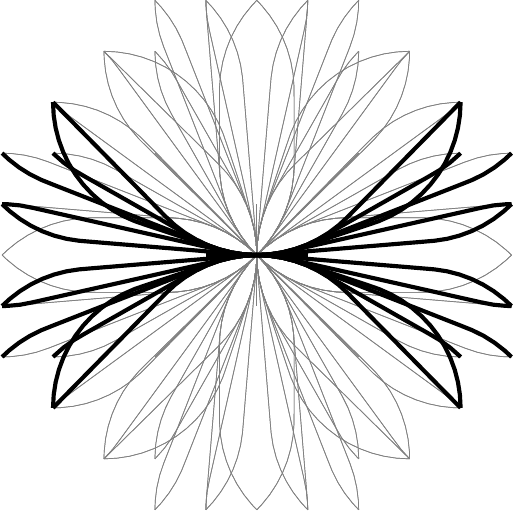}}
		\subcaption{$2\,\mathrm{m}$ grid}
	\end{subfigure}
	\begin{subfigure}[b]{0.49\columnwidth}
		\centering
		\scalebox{0.7}{\includegraphics{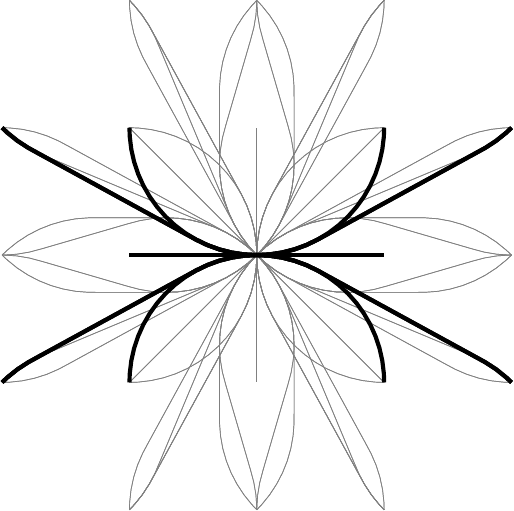}}
		\subcaption{$5\,\mathrm{m}$ grid}
	\end{subfigure}
	\caption{The motion primitive sets, used for the state lattice approaches 
		in Table \ref{table:city_results_normal} and 
		\ref{table:city_results_big}, respectively.}
	\label{fig:primitives}
\end{figure}
\section{Results and Discussion}
\label{sec:discussion}
We collect the evaluation's results in Tables 
\ref{table:city_results}-\ref{table:randommaps}. The tables' scores 
report how often an approach (on the table row) 
generates a better solution than another (on the table column). 
Whenever an approach is better, the score was changed by $+1$, by $0$ for a 
draw (i.e., both fail to find a solution), and $-1$ if the approach yielded the 
higher cost. Results are shown for different ratios $\eta=\rho / 
w$. In Table \ref{table:city_results}, a green cell highlights that the 
approach on the table row has a better performance of the one indicated on the 
table column, red otherwise. Table \ref{table:randommaps} reports only the 
scores obtained by Dispertio against the baselines.

\begin{figure}[t!]
	\centering
	\begin{subfigure}[b]{0.32\columnwidth}
		\centering
		\includegraphics[width=\columnwidth]{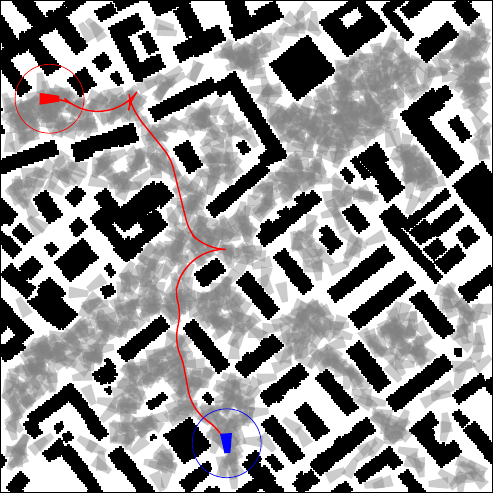}
		\subcaption{i.i.d.}
	\end{subfigure}
	\begin{subfigure}[b]{0.32\columnwidth}
		\centering
		\includegraphics[width=\columnwidth]{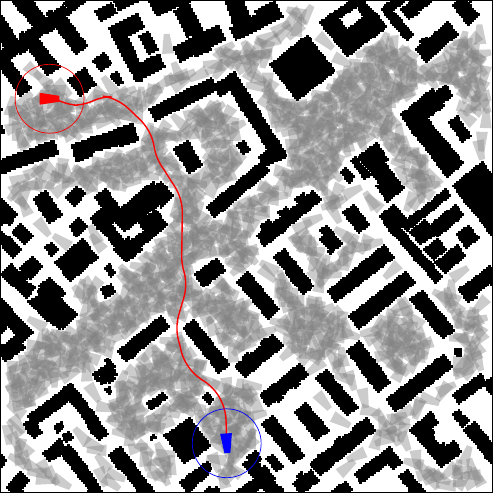}
		\subcaption{Halton}
	\end{subfigure}
	\begin{subfigure}[b]{0.32\columnwidth}
		\centering
		\includegraphics[width=\columnwidth]{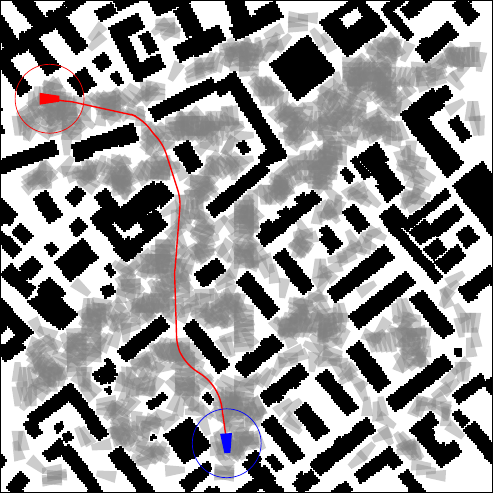}
		\subcaption{Dispertio}
	\end{subfigure}
	\begin{subfigure}[b]{0.49\columnwidth}
		\centering
		\includegraphics[width=\columnwidth]{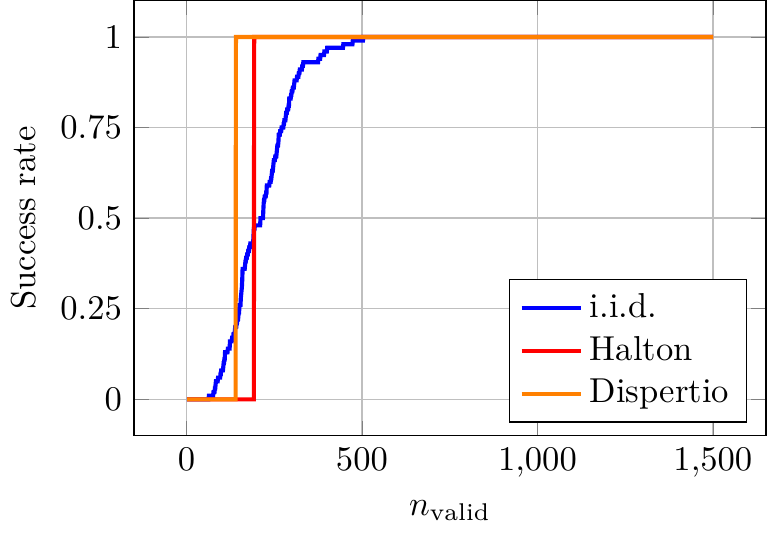}
		\subcaption{Success rate}
	\end{subfigure}
	\begin{subfigure}[b]{0.49\columnwidth}
		\centering
		\includegraphics[width=\columnwidth]{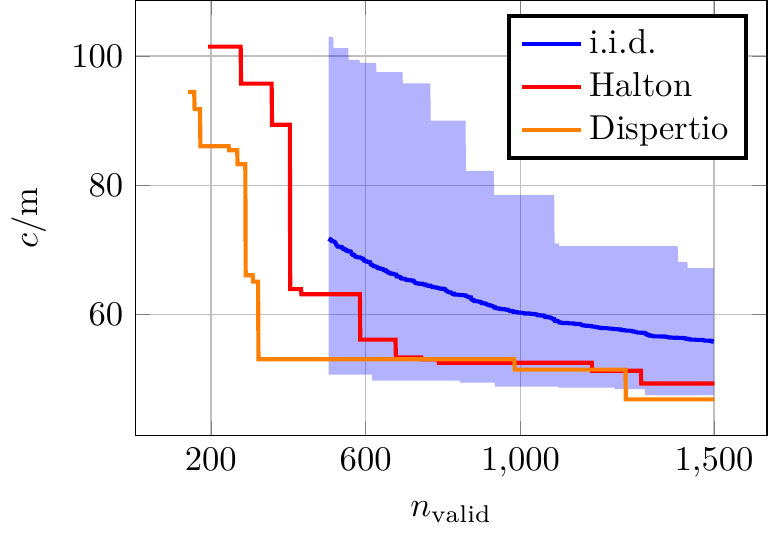}
		\subcaption{Cost}
	\end{subfigure}
	\caption{Qualitative comparison of i.i.d., Halton and Dispertio. The top 
	row shows example paths obtained after $1500$ valid samples connecting 
		starts (in \textbf{red}) with goals (in \textbf{blue}). The gray 	
		footprints represent the roadmap's vertices. The bottom row shows 
		success rate and cost for this example.}
	\label{fig:qualitativeresults}
\end{figure}
\emph{Experiment 1)}~Table \ref{table:city_results} shows the results of the 
first experiment considering 13 maps with 50 queries each (i.e., a total of 650 
planning queries, thus results have been normalized with 
$\sigma_\mathrm{rand}=\sqrt{650}$). 
Overall our approach achieves better costs and a higher 
success rate compared to the baselines. In general, Halton is the second best placed. 
State-lattice performs poorly, indicating that more effort is required for 
their motion primitives design (i.e., possibly also using a notion of dispersion in control space).
Fig.~\ref{fig:qualitativeresults} shows an example planning query for i.i.d., 
Halton sampling and our approach. It reports the obtained paths, the trend for the success rate and the cost progression. 
The blue range shows the minimum and maximum cost observed in these 
runs for the i.i.d.\ sampler. The cost results are only shown for success rates 
of $100\%$.
Cost and success rate progressions of Fig.~\ref{fig:qualitativeresults} show 
how our approach is faster in getting an initial good solution, and faster (as the number of samples increases) in converging to lower cost solutions in those cluttered and narrow scenarios.

\emph{Experiment 2)}~ Table \ref{table:randommaps} shows 
the results for the second experiment. They are normalized by 
$\sigma_\mathrm{rand}=\sqrt{1000}$, with 1000 being the the amount of random 
planning queries. Also for the second experiment, our approach achieves better 
performance in terms of final cost solution even in higher dimensional spaces. 
Furthermore, in very cluttered environments (with $\eta=1.0$) our approach  
achieves on average a 10\% higher success rate than the baselines, indicating how it can better exploit the knowledge of the nonholonomic constraints (i.e., maneuvering capabilities) in narrower scenarios. 

Moreover as reported in Fig.~\ref{fig:dispersion_progress}, Dispertio has a lower dispersion value at each iteration (i.e., numbers 
of samples) than the baselines, mainly due to the fact that it better exploits the knowledge of the system dynamics (i.e., by using the 
steer function and the reachable set computation).
\begin{table}[t]
	\centering
	\begin{subtable}{\columnwidth}
		\centering
		\begin{tabular}{*{6}{c}}
		\hline
		&  i.i.d.  &  Halton  &  Poccia  & {Dispertio} &  Lattice \\
		\hline
		i.i.d.  & $\cdot$  & \cca{-3.41}  & \cca{-3.06}  & \cca{-5.33}  & \cca{21.93} \\
		Halton  & \cca{3.41}  & $\cdot$  & \cca{-0.08}  & \cca{-1.26}  & \cca{22.28} \\
		Poccia  & \cca{3.06}  & \cca{0.078}  & $\cdot$  & \cca{-1.06}  & \cca{21.53} \\
		\textbf{Dispertio}   & \cca{5.33}  & \cca{1.26}  & \cca{1.06}  & $\cdot$  & \cca{22.75} \\
		Lattice & \cca{-21.93}  & \cca{-22.28}  & \cca{-21.53}  & \cca{-22.75}  & $\cdot$ \\
		\hline
		\end{tabular}
		\subcaption{PRM*, Reeds-Shepp, $\eta=0.1$, $n_\mathrm{all}=5000$}\label{table:city_results_normal}
	\end{subtable}
	\begin{subtable}{\columnwidth}
		\centering
		\begin{tabular}{*{6}{c}}
			&  i.i.d.  &  Halton  &  Poccia  & {Dispertio}  &  Lattice \\
			\hline
			i.i.d.  & $\cdot$  & \cca{-4.79}  & \cca{-4.04}  & \cca{-9.18}  & \cca{21.69} \\
			Halton  & \cca{4.79}  & $\cdot$  & \cca{1.26}  & \cca{-4.31}  & \cca{23.3} \\
			Poccia & \cca{4.04}  & \cca{-1.26}  & $\cdot$  & \cca{-4.98}  & \cca{23.06} \\
			\textbf{Dispertio} & \cca{9.18}  & \cca{4.31}  & \cca{4.98}  & $\cdot$  & \cca{23.65} \\
			Lattice & \cca{-21.69}  & \cca{-23.3}  & \cca{-23.06}  & \cca{-23.65}  & $\cdot$ \\
		\end{tabular}
		\subcaption{PRM*, Reeds-Shepp, $\eta=0.05$, $n_\mathrm{all}=3200$}\label{table:city_results_big}
	\end{subtable}
	\caption{Path quality results of all the methods for the city-maps 
	benchmark \cite{sturtevant2012benchmarks}. Dispertio obtains on average 
	better solutions against all the baselines.}\label{table:city_results}
\end{table}
\begin{table}[t]
	\centering
	\begin{subtable}{\columnwidth}
		\centering
		\begin{tabular}{l*{3}{c}}
		&  i.i.d.  &  Halton  &  Poccia \\
		\hline
		RS, $\eta=1$, $n_\mathrm{all}=1500$  &  \cca{4.87}  & \cca{4.11}  & \cca{3.1}  \\
		RS, $\eta=0.25$, $n_\mathrm{all}=1500$ & \cca{6.26}  & \cca{4.08}  & \cca{1.20} \\
		RS, $\eta=0.1$, $n_\mathrm{all}=1500$  & \cca{10.69}  & \cca{6.7}  & \cca{2.47} \\
		RS, $\eta=0.05$, $n_\mathrm{all}=1500$  & \cca{11.61}  & \cca{7.15}  & \cca{4.59} \\
		KC, $\{[-\pi,\pi]/\sim\}^6$, $n_\mathrm{all}=30000$  & \cca{2.28}  & \cca{1.01}  & $\cdot$ \\
		KC, $[-3,3]^6$, $n_\mathrm{all}=30000$  & \cca{7.78}  & \cca{7.78}  & $\cdot$ \\
		\end{tabular}
	\end{subtable}
	\caption{Path quality performance of Dispertio against the baselines, on 
	randomized maps and queries in different 
	spaces (Reed-Shepp and 6D Kinematic Chain).}
	\label{table:randommaps}
\end{table}

\section{Conclusions}
\label{sec:conclusions}
In this work we extend deterministic sampling-based 
motion planning to the class of symmetric and optimal driftless systems, by 
proposing Dispertio, an algorithm for optimized deterministic 
sampling set generation. When used in combination with PRM*, we 
prove that the approach is deterministically complete and retains asymptotic 
optimality. In the evaluation, we show that our sampling technique outperforms 
state-of-the-art methods in terms of solution cost and planning efficiency, 
while also converging faster to lower cost solutions. As future work, we are 
interested in extending the approach towards non-uniform sampling schemes, for example to exploit learned priors, and to systems with drift. 

\section*{Acknowledgment}
This work has been partly funded from the European Union's Horizon 2020 
research and innovation programme under grant agreement No 732737 (ILIAD).

\bibliographystyle{plain}
\footnotesize
\bibliography{references}
\end{document}